\documentclass{article}

\usepackage[final]{nips_2016}

\usepackage[utf8]{inputenc} % allow utf-8 input
\usepackage[T1]{fontenc}    % use 8-bit T1 fonts
\usepackage{hyperref}       % hyperlinks
\usepackage{url}            % simple URL typesetting
\usepackage{booktabs}       % professional-quality tables
\usepackage{amsfonts}       % blackboard math symbols
\usepackage{nicefrac}       % compact symbols for 1/2, etc.
\usepackage{microtype}      % microtypography
\usepackage{amsmath,amsthm}
\usepackage{graphicx}
\usepackage{caption}
\usepackage{subcaption}
\usepackage{algorithm2e}

\newtheorem{proposition}{Proposition}

\usepackage{tikz}
\usetikzlibrary{backgrounds}
\usetikzlibrary{calc}
\usetikzlibrary{fit}
\usetikzlibrary{positioning}

\pgfdeclarelayer{background}
\pgfsetlayers{background,main}

\def\xx{{\mathbf x}}

\def\zz{{\mathbf z}}

\definecolor{mygreen}{HTML}{167dde}
\definecolor{myred}{HTML}{f22835}

\colorlet{greenfill}{mygreen!20!white}
\colorlet{redfill}{myred!20!white}
\colorlet{moreredfill}{myred!40!white}

\title{GibbsNet: Iterative Adversarial Inference for Deep Graphical Models}

\begin{document}

\author {Alex Lamb \\ MILA, Universite de Montreal \\ lambalex@iro.umontreal.ca \And R Devon Hjelm \\ MILA, Universite de Montreal \\ erroneus@gmail.com \And Yaroslav Ganin \\ MILA, Universite de Montreal \\ yaroslav.ganin@gmail.com \And Joseph Paul Cohen \\ MILA, Universite de Montreal \\ Institute for Reproducible Research \\ joseph@josephpcohen.com \And Aaron Courville \\ MILA, Universite de Montreal \\ CIFAR  \\  aaron.courville@gmail.com \And Yoshua Bengio \\ MILA, Universite de Montreal \\ CIFAR \\ yoshua.umontreal@gmail.com}

\maketitle

\begin{abstract}

%New abstract: 

Directed latent variable models that formulate the joint distribution as $p(x,z) = p(z) p(x \mid z)$ have the advantage of fast and exact sampling. However, these models have the weakness of needing to specify $p(z)$, often with a simple fixed prior that limits the expressiveness of the model.  Undirected latent variable models discard the requirement that $p(z)$ be specified with a prior, yet sampling from them generally requires an iterative procedure such as blocked Gibbs-sampling that may require many steps to draw samples from the joint distribution $p(x, z)$.  We propose a novel approach to learning the joint distribution between the data and a latent code which uses an adversarially learned iterative procedure to gradually refine the joint distribution, $p(x, z)$, to better match with the data distribution on each step.  GibbsNet is the best of both worlds both in theory and in practice.  Achieving the speed and simplicity of a directed latent variable model, it is guaranteed (assuming the adversarial game reaches the virtual training criteria global minimum) to produce samples from $p(x, z)$ with only a few sampling iterations.  Achieving the expressiveness and flexibility of an undirected latent variable model, GibbsNet does away with the need for an explicit $p(z)$ and has the ability to do attribute prediction, class-conditional generation, and joint image-attribute modeling in a single model which is not trained for any of these specific tasks.  We show empirically that GibbsNet is able to learn a more complex $p(z)$ and show that this leads to improved inpainting and iterative refinement of $p(x, z)$ for dozens of steps and stable generation without collapse for thousands of steps, despite being trained on only a few steps.  

\end{abstract}

\section{Introduction}

Generative models are powerful tools for learning an underlying representation of complex data. 
% YB I don't like that sentence so I comment it
%While often not as powerful as supervised methods for well-designed tasks, the potential of undirected models lie in being more general in representation and comprehensive in application. 
While early undirected models, such as Deep Boltzmann Machines or DBMs~\citep{salakhutdinov2009deep}, showed great promise, practically they did not scale well to complicated high-dimensional settings (beyond MNIST), possibly because of optimization and mixing difficulties~\citep{bengio2012mixing}.  More recent work on Helmholtz machines~\citep{bornschein2015bidirection} and on variational autoencoders~\citep{kingma2013auto} borrow from deep learning tools and can achieve impressive results, having now been adopted in a large array of domains~\citep{larsen2015auto}.

\begin{figure}[ht]
\centering
\begin{tikzpicture}[
    black!50, text=black,
    node distance=4mm,
    dnode/.style={
        align=center,
        % The shape:
        rectangle,minimum size=10mm,rounded corners,
        % The rest
        inner sep=5pt,
        left color=redfill,
        right color=greenfill},
    rnode/.style={
        align=center,
        % The shape:
        rectangle,
        minimum width=27mm,
        minimum height=10mm,
        rounded corners,
        % The rest
        inner sep=5pt,
        very thick,draw=black!50},
    tuplenode/.style={
        align=center,
        % The shape:
        rectangle,minimum size=10mm,rounded corners,
        % The rest
        inner sep=5pt},
    darrow/.style={
        rounded corners,-latex,shorten <=5pt,shorten >=1pt,line width=2mm}]

\matrix[row sep=20mm,column sep=8mm] {
    \node (z_0) [rnode] {$ \zz_0 \sim \mathcal{N}(0,I) $}; &
    \node (xt_i) [rnode,dotted] {$ \xx_i \sim p(\xx \, | \, \zz_i) $}; &
    \node (z_n) [rnode] {$ \zz_N \sim q(\zz \, | \, \xx_{N - 1}) $}; &
    \node (xt_n) [rnode] {$ \xx_N \sim p(\xx \, | \, \zz_N) $}; \\

    & 
    \node (z_ipo) [rnode,dotted] {$ \zz_{i + 1} \sim q(\zz \, | \, \xx_i) $}; &
    \node (zh_0) [rnode] {$ \hat{\zz} \sim q(\zz \, | \, \xx_{data}) $}; &
    \node (x_0) [rnode] {$ \xx_{data} \sim q(\xx) $}; \\
};

\coordinate (upper_mid) at ($ (z_n.center)!0.5!(xt_n.center) $);
\coordinate (lower_mid) at ($ (zh_0.center)!0.5!(x_0.center) $);
\coordinate (d_pos) at ($ (upper_mid)!0.5!(lower_mid) $);

\node (D) at (d_pos) [dnode] {$ D(\zz, \xx) $};

\begin{pgfonlayer}{background}
    \node (upper_tuple) [fit=(z_n) (xt_n),tuplenode,fill=redfill] {};
    \node (lower_tuple) [fit=(zh_0) (x_0),tuplenode,fill=greenfill] {};
\end{pgfonlayer}
\node (cycle) [fit=(xt_i) (z_ipo),rnode,dotted] {};

\draw[-latex,shorten >=6pt,very thick,dotted] (z_0) -- (xt_i);
\draw[-latex,shorten <=6pt,shorten >=1pt,very thick,dotted] (xt_i) -- (z_n);

\draw[-latex,shorten >=1pt,very thick] (z_n) -- (xt_n);
\draw[-latex,shorten >=1pt,very thick] (x_0) -- (zh_0);

\draw[darrow,draw=redfill] (z_n) |- (D);
\draw[darrow,draw=greenfill] (x_0) |- (D);

\draw[-latex,shorten <=1pt,shorten >=1pt,very thick,dotted] (xt_i) to [bend left=45] (z_ipo);
\draw[-latex,shorten <=1pt,shorten >=1pt,very thick,dotted] (z_ipo) to [bend left=45] (xt_i);

\end{tikzpicture}
\caption{Diagram illustrating the training procedure for GibbsNet. The \textbf{unclamped chain} (\textcolor{black!50}{dashed box}) starts with a sample from an isotropic Gaussian distribution $ \mathcal{N}(0, I) $ and runs for $ N $ steps. The last step (iteration $N$) shown as a \textcolor{myred}{solid pink box} is then compared with a single step from the \textbf{clamped chain} (\textcolor{mygreen}{solid blue box}) using joint discriminator $ D $.}
\label{fig:diagram}
\end{figure}
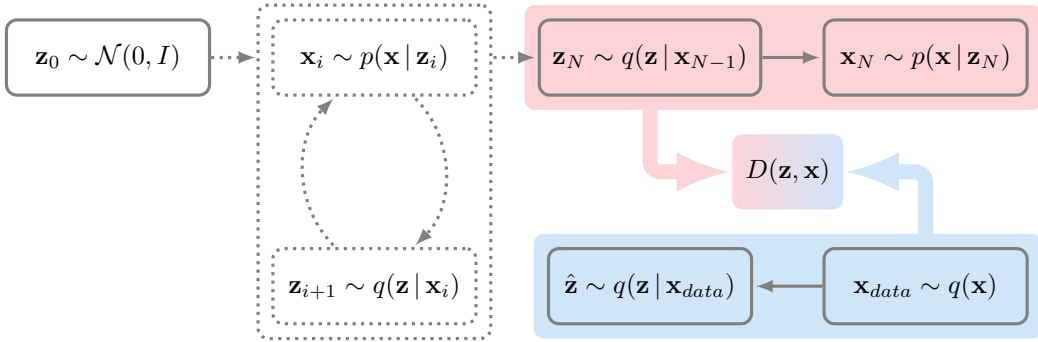

Many of the important generative models available to us rely on a formulation of some sort of stochastic latent or hidden variables along with a generative relationship to the observed data. Arguably the simplest is the \emph{directed graphical models} (such as the VAE) with a factorized decomposition $p(z, x) = p(z) p(x \mid z)$. In this, it is typical to assume that $p(z)$ follows some factorized prior with simple statistics (such as Gaussian). While sampling with directed models is simple, inference and learning tends to be difficult and often requires advanced techniques such as approximate inference using a proposal distribution for the true posterior. 

The other dominant family of graphical models are \emph{undirected graphical models}, such that the joint is represented by a product of clique potentials and a normalizing factor. It is common to assume that the clique potentials are positive, so that the un-normalized density can be represented by an energy function, $E$ and the joint is represented by $p(x,z) = e^{-E(z, x)} / Z$, where $Z$ is the normalizing constant or partition function. These so-called energy-based models (of which the Boltzmann Machine is an example) are potentially very flexible and powerful, but are difficult to train in practice and do not seem to scale well. Note also how in such models, the marginal $p(z)$ can have a very rich form (as rich as that of $p(x)$).

The methods above rely on a fully parameterized joint distribution (and approximate posterior in the case of directed models), to train with approximate maximum likelihood estimation~\citep[MLE,][]{dempster1977maximum}. Recently, generative adversarial networks~\citep[GANs,][]{goodfellow2014generative} have provided a likelihood-free solution to generative modeling that provides an implicit distribution unconstrained by density assumptions on the data. In comparison to MLE-based latent variable methods, generated samples can be of very high quality~\citep{radford2015dcgan}, and do not suffer from well-known problems associated with parameterizing noise in the observation space~\citep{goodfellow2016nips}. Recently, there have been advances in incorporating latent variables in generative adversarial networks in a way reminiscent of Helmholtz machines~\citep{dayan1995helmholtz}, such as adversarially learned inference~\citep{dumoulin2016ali, donahue2016bigan} and implicit variational inference~\citep{huszar2017implicit}.

These models, as being essentially complex directed graphical models, rely on approximate inference to train. While potentially powerful, there is good evidence that using an approximate posterior necessarily limits the generator in practice \citep{hjelm2016iterative,rezende2015variational}. In contrast, it would perhaps be more appropriate to start with inference (encoder) and generative (decoder) processes and derive the prior directly from these processes. This approach, which we call GibbsNet, uses these two processes to define a transition operator of a Markov chain similar to Gibbs sampling, alternating between sampling observations and sampling latent variables. This is similar to the previously proposed generative stochastic networks~\citep[GSNs,][]{bengio2013gsn} but with a GAN training framework rather than minimizing reconstruction error. By training a discriminator to place a decision boundary between the data-driven distribution (with $x$ clamped) and the free-running model (which alternates between sampling $x$ and $z$), we are able to train the model so that the two joint distributions $(x, z)$ match. This approach is similar to Gibbs sampling in undirected models, yet, like traditional GANs, it lacks the strong parametric constraints, i.e., there is no explicit energy function. While losing some the theoretical simplicity of undirected models, we gain great flexibility and ease of training. In summary, our method offers the following contributions: 

\begin{itemize}

    \item We introduce the theoretical foundation for a novel approach to learning and performing inference in deep graphical models. The resulting model of our algorithm is similar to undirected graphical models, but avoids the need for MLE-based training and also lacks an explicitly defined energy, instead being trained with a GAN-like discriminator.
    
    %\item The ability to handle features from the samples being missing, either during training or during evaluation.  The distribution of what features are missing does not need to be the same during training and evaluation.  
    
     \item We present a stable way of performing inference in the adversarial framework, meaning that useful inference is performed under a wide range of architectures for the encoder and decoder networks.  This stability comes from the fact that the encoder $q(z \mid x)$ appears in both the clamped and the unclamped chain, so gets its training signal from both the discriminator in the clamped chain and from the gradient in the unclamped chain.  
    
    \item We show improvements in the quality of the latent space over models which use a simple prior for $p(z)$.  This manifests itself in improved conditional generation.  The expressiveness of the latent space is also demonstrated in cleaner inpainting, smoother mixing when running blocked Gibbs sampling, and better separation between classes in the inferred latent space.  
    
    \item Our model has the flexibility of undirected graphical models, including the ability to do label prediction, class-conditional generation, and joint image-label generation in a single model which is not explicitly trained for any of these specific tasks.  To our knowledge our model is the first model which combines this flexibility with the ability to produce high quality samples on natural images.  

\end{itemize}

\section{Proposed Approach: GibbsNet}

The goal of GibbsNet is to train a graphical model with transition operators that are defined and learned directly by matching the joint distributions of the model expectation with that with the observations clamped to data. 
This is analogous to and inspired by undirected graphical models, except that the transition operators, which correspond to blocked Gibbs sampling, are defined to move along a defined energy manifold, so we will make this connection throughout our formulation.

We first explain GibbsNet in the simplest case where the graphical model consists of a single layer of observed units and a single layer of latent variable with stochastic mappings from one to the other as parameterized by arbitrary neural network.
Like Professor Forcing~\citep{lamb2016professor}, % see ml.bib
GibbsNet uses a GAN-like discriminator to make two distributions match, one corresponding to the model iteratively sampling both observation, $x$, and latent variables, $z$ (free-running), and one corresponding to the same generative model but with the observations, $x$, clamped. 
The free-running generator is analogous to Gibbs sampling in Restricted Boltzmann Machines~\citep[RBM,][]{hinton2006fast} or Deep Boltzmann Machines~\citep[DBM,][]{salakhutdinov2009deep}. % see ml.bib
In the simplest case, the free-running generator is defined by conditional distributions $q(z|x)$ and $p(x|z)$ which stochastically map back and forth between data space $x$ and latent space $z$.  

To begin our free-running process, we start the chain with a latent variable sampled from a normal distribution: $z \sim \mathcal{N}(0,I)$, and follow this by $N$ steps of alternating between sampling from $p(x|z)$ and $q(z|x)$.  
For the clamped version, we do simple ancestral sampling from $q(z|x)$, given $x_{data}$ is drawn from the data distribution $ q(x) $ (a training example).
When the model has more layers (e.g., a hierarchy of layers with stochastic latent variables, \`a la DBM), the data-driven model also needs to iterate to correctly sample from the joint. 
While this situation highly resembles that of undirected graphical models, GibbsNet is trained adversarially so that its free-running generative states become indistinguishable from its data-driven states.
In addition, while in principle undirected graphical models need to either start their chains from data or sample a very large number of steps, we find in practice GibbsNet only requires a very small number of steps (on the order of 3 to 5 with very complex datasets) from noise.

% DH: we don't show this and this sentence is a bit confusing
% As we show below, this also results in making GibbsNet learn a transition operator, as the composition of $p$ and $q$, i.e., the repeated application of this composition of $p$ and $q$ converges to a stationary distribution and the training procedure also pulls this stationary distribution towards the data distribution. 

%The foundation of the GibbsNet algorithm is that we have two distributions $p(x,z)$ and $q(x,z)$, which share all parameters and only differ in that $q(x,z)$ has $x$ clamped to observed samples from the data distribution.  The objective in GibbsNet is to iteratively sample from both $p$ and $q$ using an alternating sampling procedure analogous to blocked Gibbs sampling in undirected graphical models, and adversarially encourage the $p(x_i,z_i)$ and $q(x_i,z_i)$ distributions to become increasingly similar as the number $i$ of alternating steps increases.  

%Aaron: are we using a hierarchy? if not I find the next sentence confusing. If we are, then it's still confusing - but the fix would be different.
%DH: we already talked about this
%In the general case where there are multiple layers (similar to the DBM, each iteration of the sampling procedure is still analogous to the corresponding blocked Gibbs sampling scheme: at each step one samples a new value for a particular layer, conditioned on all layers that it is connected to.  

An example of the free-running (unclamped) chain can be seen in Figure~\ref{fig:svhn_chains}.
An interesting aspect of GibbsNet is that we found that it was enough and in fact best experimentally to back-propagate discriminator gradients {\em through a single step of the iterative procedure}, yielding more stable training. An intuition for why this helps is that each step of the procedure is supposed to generate increasingly realistic samples.  However, if we passed gradients through the iterative procedure, then this gradient could encourage the earlier steps to store features which have downstream value instead of immediate realistic $x$-values.

%\begin{figure}[t]
%\centering
%\input{figures/tikz/diagram_old}
%\caption{Diagram illustrating the training procedure for GibbsNet.  The end of the unclamped chain (red) is %compared with the clamped chain (green) using a joint discriminator.  }
%\label{fig:diagram_old}
%\end{figure}

%\subsection{Graphical Models}

%While our method is tractable for any undirected graphs with a small chromatic number, we consider and experimentally explore three types of undirected graphs in this paper, shown in Fig.~\ref{fig:diagram}.  

%\begin{figure}[ht]
%\centering
%\includegraphics[width=0.5\textwidth]{undirected_graph_illustration.png}
%\caption {Illustration of the three types of undirected graphs considered in this paper. $x$ is an input, $z$ or $z_i$ are latent variables, and $y$ is a target label.}
%\label{fig:graphs}
%\end{figure}

\begin{figure}[ht]
\centering

\includegraphics[width=\textwidth]{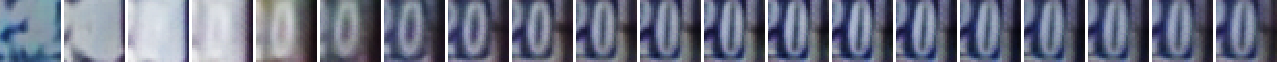}
\includegraphics[width=\textwidth]{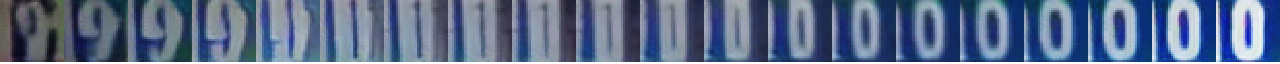}
\includegraphics[width=\textwidth]{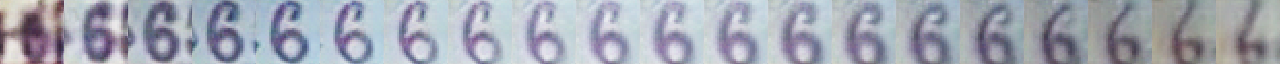}
\includegraphics[width=\textwidth]{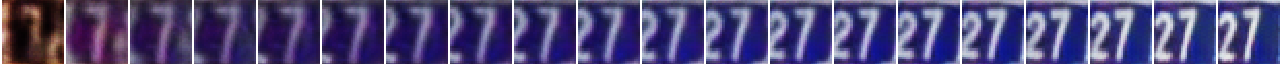}
\vspace{5pt}

\caption{Evolution of samples for 20 iterations from the unclamped chain, trained on the SVHN dataset starting on the left and ending on the right.}
\label{fig:svhn_chains}
\end{figure}

\subsection{Theoretical Analysis}

We consider a simple case of an undirected graph with single layers of visible and latent units trained with alternating 2-step ($p$ then $q$) unclamped chains and the asymptotic scenario where the GAN objective is properly optimized.
We then ask the following questions: in spite of training for a bounded number of Markov chain steps, are we learning a transition operator? Are the encoder and decoder estimating compatible conditionals associated with the stationary distribution of that transition operator? We find positive answers to both questions.  

A high level explanation of our argument is that if the discriminator is fooled, then the consecutive $(z,x)$ pairs from the chain match the data-driven $(z,x)$ pair. Because the two marginals on $x$ from these two distributions match, we can show that the next $z$ in the chain will form again the same joint distribution. Similarly, we can show that the next $x$ in the chain also forms the same joint with the previous $z$. Because the state only depends on the previous value of the chain (as it's Markov), then all following steps of the chain will also match the clamped distribution.  This explains the result, validated experimentally, that even though we train for just a few steps, we can generate high quality samples for thousands or more steps.  

\begin{proposition}
\label{prop:stationary}
If \textbf{(a)} the stochastic encoder $q(z|x)$ and stochastic decoder $p(x|z)$ inject noise such that the transition operator defined by their composition ($p$ followed by $q$ or vice-versa) allows for all possible $x$-to-$x$ or $z$-to-$z$ transitions ($x \rightarrow z \rightarrow x$ or $z \rightarrow x \rightarrow z$), and if 
\textbf{(b)} those GAN objectives are properly trained in the sense that the discriminator is fooled in spite of having sufficient capacity and training time, then \textbf{(1)} the Markov chain which alternates the stochastic encoder followed by the stochastic decoder as its transition operator $T$ (or vice-versa) has the data-driven distribution $\pi_D$ as its stationary distribution $\pi_T$, \textbf{(2)} the two conditionals $q(z|x)$ and $p(x|z)$ converge to compatible conditionals associated with the joint $\pi_D=\pi_T$.
\end{proposition}
\begin{proof}
When the stochastic decoder and encoder inject noise so that their composition forms a transition operator $T$ with paths with non-zero probability from any state to any other state, then $T$ is ergodic. So condition (a) implies that $T$ has a stationary distribution $\pi_T$. 
 The properly trained GAN discriminators for each of these two steps (condition (b)) forces the matching of the distributions of the pairs $(z_t,x_t)$ (from the generative trajectory) and $(x , z)$ with $x \sim q(x)$, the data distribution and $z \sim q(z \mid x)$, both pairs converging to the same data-driven distribution $\pi_D$. Because $(z_t,x_t)$
 has the same joint distribution as $(z, x)$, it means that $x_t$ has the same distribution as $x$. Since $z  \sim q(z \mid x)$,
 when we apply $q$ to $x_t$, we get $z_{t+1}$ which must form a joint $(z_{t+1},x_t)$ which has the same distribution as $(z,x)$. Similarly, since we just showed that $z_{t+1}$ has the same distribution as $z$ and thus the same as $z_t$, if we apply $p$ to $z_{t+1}$, we get $x_{t+1}$ and the joint $(z_{t+1},x_{t+1})$ must have the same distribution as $(z,x)$.
 Because the two pairs $(z_t,x_t)$ and $(z_{t+1},x_{t+1})$ have the same joint distribution $\pi_D$, it means that the transition operator $T$, that maps samples $(z_t,x_t)$ to samples $(z_{t+1},x_{t+1})$, maps $\pi_D$ to itself, i.e., $\pi_D=\pi_T$ is both the data distribution and the stationary distribution of $T$ and result (1) is obtained.
Now consider the "odd" pairs $(z_{t+1},x_t)$ and $(z_{t+2},x_{t+1})$ in the generated sequences. Because of (1), $x_{t}$ and $x_{t+1}$ have the same marginal distribution $\pi_D(x)$. Thus when we apply the same $q(z|x)$ to these $x$'s we obtain that 
$(z_{t+1},x_t)$ and $(z_{t+2},x_{t+1})$ also have the same distribution. Following the same reasoning as for proving (1), we conclude that the associated transition operator $T_{\rm odd}$ has also $\pi_D$ as stationary distribution. 
So starting from $z \sim \pi_D(z)$ and applying $p(x \mid z)$ gives an $x$ so that the pair $(z,x)$ has $\pi_D$ as joint distribution, i.e., $\pi_D(z,x) = \pi_D(z) p(x \mid z)$. This means that $p(x \mid z)=\frac{\pi_D(x,z)}{\pi_D(z)}$ is the $x \mid z$ conditional of $\pi_D$. Since $(z_t, x_t)$ also converges to joint distribution $\pi_D$, we can apply the same argument when starting from an $x \sim \pi_D(x)$ followed by $q$ and obtain that $\pi_D(z,x) = \pi_D(x) q(z \mid x)$ and so $q(z|x)=\frac{\pi_D(z,x)}{\pi_D(x)}$ is the $z \mid x$ conditional of $\pi_D$. This proves result (2).
\end{proof}

\subsection{Architecture}

GibbsNet always involves three networks: the inference network $q(z|x)$, the generation network $p(x|z)$, and the joint discriminator.  In general, our architecture for these networks closely follow ~\citet{dumoulin2016ali}, except that we use the boundary-seeking GAN \citep[BGAN,][]{hjelm2017boundary} as it explicitly optimizes on matching the opposing distributions (in this case, the model expectation and the data-driven joint distributions), allows us to use discrete variables where we consider learning graphs with labels or discrete attributes, and worked well across our experiments.  

\section{Related Work}

\paragraph{Energy Models and Deep Boltzmann Machines}
The training and sampling procedure for generating from GibbsNet is very similar to that of a deep Boltzmann machine~\citep[DBM,][]{salakhutdinov2009deep}: both involve blocked Gibbs sampling between observation- and latent-variable layers.  
A major difference is that in a deep Boltzmann machine, the ``decoder" $p(x|z)$ and ``encoder" $p(z|x)$ exactly correspond to conditionals of a joint distribution $p(x,z)$, which is parameterized by an energy function.
This, in turn, puts strong constraints on the forms of the encoder and decoder.  

In a restricted Boltzmann machine~\citep[RBM,][]{hinton2010practical}, the visible units are conditionally independent given the hidden units on the adjacent layer, and likewise the hidden units are conditionally independent given the visible units.  
This may force the layers close to the data to need to be nearly deterministic, which could cause poor mixing and thus make learning difficult.  
These conditional independence assumptions in RBMs and DBMs have been discussed before in the literature as a potential weakness in these models \citep{bengio2012mixing}.

In our model, $p(x|z)$ and $q(z|x)$ are modeled by separate deep neural networks with no shared parameters.
The disadvantage is that the networks are over-parameterized, but this has the added flexibility that these conditionals can be much deeper, can take advantage of all the recent advances in deep architectures, and have fewer conditional independence assumptions than DBMs and RBMs.  

%We can show that, when its adversarial game reaches a fixed point, GibbsNet can be seen as implicitly learning a single energy function which is consistent across all steps after the burn-in period. 
%However this energy function is only an implicit consequence of running the transition operators, and we have no explicit way of computing its value, which also means that the energy is not constrained to correspond to a simple parametric form.  

\paragraph{Generative Stochastic Networks}
Like GibbsNet, generative stochastic networks~\citep[GSNs,][]{bengio2013gsn} also directly parameterizes a transition operator of a Markov chain using deep neural networks.
However, GSNs and GibbsNet have completely different training procedures.  
In GSNs, the training procedure is based on an objective that is similar to de-noising autoencoders~\citep{vincent2008extracting}. 

GSNs begin by drawing a sampling from the data, iteratively corrupting it, then learning a transition operator which de-noises it (i.e., reverses that corruption), so that the reconstruction after $k$ steps is brought closer to the original un-corrupted input.  

In GibbsNet, there is no corruption in the visible space, and the learning procedure never involves ``walk-back" (de-noising) towards a real data-point.  Instead, the processes from and to data are modeled by different networks, with the constraint of the marginal, $p(x)$, matches the real distribution imposed through the GAN loss on the joint distributions from the clamped and unclamped phases.

\paragraph{Non-Equilibrium Thermodynamics}
The Non-Equilibrium Thermodynamics method \citep{dickstein2015noneq} learns a reverse diffusion process against a forward diffusion process which starts from real data points and gradually injects noise until the data distribution matches a analytically tractible / simple distribution. This is similar to GibbsNet in that generation involves a stochastic process which is initialized from noise, but differs in that Non-Equilibrium Thermodynamics is trained using MLE and relies on noising + reversal for training, similar to GSNs above.

\paragraph{Generative Adversarial Learning of Markov Chains}
The Adversarial Markov Chain algorithm \citep[AMC,][]{song2017markov} learns a markov chain over the data distribution in the visible space.  GibbsNet and AMC are related in that they both involve adversarial training and an iterative procedure for generation.  However there are major differences.  GibbsNet learns deep graphical models with latent variables, whereas the AMC method learns a transition operator directly in the visible space.  The AMC approach involves running chains which start from real data points and repeatedly apply the transition operator, which is different from the clamped chain used in GibbsNet.  The experiments shown in Figure~\ref{fig:latent_disc} demonstrate that giving the latent variables to the discriminator in our method has a significant impact on inference.  

\paragraph{Adversarially Learned Inference (ALI)}
 Adversarially learned inference~\citep[ALI,][]{dumoulin2016ali} learns to match distributions generative and inference distributions, $p(x, z)$ and $q(x, z)$ (can be thought of forward and backward models) with a discriminator, so that $p(z) p(x \mid z) = q(x) q(z \mid x)$.  
In the single latent layer case, GibbsNet also has forward and reverse models, $p(x \mid z)$ and $q(z \mid x)$.  The un-clamped chain is sampled as $p(z), p(x \mid z), q(z \mid x), p(x \mid z), \dots$ and the clamped chain is sampled as $q(x), q(z \mid x)$.  We then adversarially encourage the clamped chain to match the equilibrium distribution of the unclamped chain.  When the number of iterations is set to $N = 1$, GibbsNet reduces to ALI.
However, in the general setting of $N > 1$, Gibbsnet should learn a richer representation than ALI, as the prior, $p(z)$, is no longer forced to be the simple one at the beginning of the unclamped phase.

\section{Experiments and Results}

The goal of our experiments is to explore and give insight into the joint distribution $p(x,z)$ learned by GibbsNet and to understand how this joint distribution evolves over the course of the iterative inference procedure.  Since ALI is identical to GibbsNet when the number of iterative inference steps is $N=1$, results obtained with ALI serve as an informative baseline.

From our experiments, the clearest result (covered in detail below) is that the $p(z)$ obtained with GibbsNet can be more complex than in ALI (or other directed graphical models).  This is demonstrated directly in experiments with $2$-D latent spaces and indirectly by improvements in classification when directly using the variables $q(z \mid x)$.  
We achieve strong improvements over ALI using GibbsNet even when $q(z \mid x)$ has exactly the same architecture in both models.  

We also show that GibbsNet allows for gradual refinement of the joint, $(x, z)$, in the sampling chain $q(z \mid x), p(x \mid z)$.
This is a result of the sampling chain making small steps towards the equilibrium distribution.  
This allows GibbsNet to gradually improve sampling quality when running for many iterations.  
Additionally it allows for inpainting and conditional generation where the conditioning information is not fixed during training, and indeed where the model is not trained specifically for these tasks.

\subsection{Expressiveness of GibbsNet's Learned Latent Variables}

%Alex: should we cut this paragraph - seems kind of extraneous.  
%Previous work on semi-supervised learning has used the features of the discriminator as the basis for performing classification.  While this has performed well, even achieving state-of-the-art results \citep{salimans2016improved,dumoulin2016ali} this approach has a major limitation: if the generator's outputs match the true data distribution exactly, then the optimal discriminator's output is constant, and it does not need to produce useful features.  This is the case for the original GAN and other GAN extentions, including the Wasserstein GAN and the Least Squares GAN \citep{arjovsky2017wgan,mao2016lsgan}. 
%With ALI, the learned latent variables from the inference network are quite poor for classification.  

\paragraph{Latent structure of GibbsNet} 
The latent variables from $q(z \mid x)$ learned from GibbsNet are more expressive than those learned with ALI.  We show this in two ways.  First, we train a model on the MNIST digits $0$, $1$, and $9$ with a $2$-D latent space which allows us to easily visualize inference.  As seen in Figure~\ref{fig:latent_disc}, we show that GibbsNet is able to learn a latent space which is not Gaussian and has a structure that makes the different classes well separated.  

\paragraph{Semi-supervised learning}
Following from this, we show that the latent variables learned by GibbsNet are better for classification. The goal here is not to show state of the art results on classification, but instead to show that the requirement that $p(z)$ be something simple (like a Gaussian, as in ALI) is undesirable as it forces the latent space to be filled. This means that different classes need to be packed closely together in that latent space, which makes it hard for such a latent space to maintain the class during inference and reconstruction.

We evaluate this property on two datasets: Street View House Number~\citep[SVHN,][]{netzer2011reading} and permutation invariant MNIST.  In both cases we use the latent features $q(z \mid x)$ directly from a trained model, and train a $2$-layer MLP on top of the latent variables, without passing gradient from the classifier through to $q(z \mid x)$.  ALI and GibbsNet were trained for the same amount of time and with exactly the same architecture for the discriminator, the generative network, $p(x \mid z)$, and the inference network, $q(z \mid x)$.  

On permutation invariant MNIST, ALI achieves $91$\% test accuracy and GibbsNet achieves $97.7$\% test accuracy.  On SVHN, ALI achieves $66.7$\% test accuracy and GibbsNet achieves $79.6$\% test accuracy.  This does not demonstrate a competitive classifier in either case, but rather demonstrates that the latent space inferred by GibbsNet keeps more information about its input image than the encoder learned by ALI.  This is consistent with the reported ALI reconstructions \citep{dumoulin2016ali} on SVHN where the reconstructed image and the input image show the same digit roughly half of the time.  

We found that ALI's inferred latent variables not being effective for classification is a fairly robust result that holds across a variety of architectures for the inference network.  For example, with $1024$ units, we varied the number of fully-connected layers in ALI's inference network between $2$ and $8$ and found that the classification accuracies on the MNIST validation set ranged from $89.4$\% to $91.0$\%.  Using $6$ layers with $2048$ units on each layer and a $256$ dimensional latent prior achieved $91.2$\% accuracy.   
This suggests that the weak performance of the latent variables for classification is due to ALI's prior, and is probably not due to a lack of capacity in the inference network.  

%We observe ALI achieves 91.0\% test accuracy when an SVM classifier is built using the latent features. Our approach obtains 97.0\%  with two steps and 97.7\% with three. The extremely weak performance of ALI's latent features for classification was also reported in \citep{dumoulin2016ali} and is consistent with our experiments on reconstruction and inpainting. 

%If we jointly train the model for classification ALI achieves 96.4\% and we still obtain an improvement by adding steps with 97.6\%  for 2 steps and 98.1\% for 3.  

%To evaluate the SVHN task we use 1000 training examples and the procedure in \citep{dumoulin2016ali}. When the latent variables of a variational autoencoder \citep{kingma2014semi} are used and result in a 54.33\% misclassification test error. 

%Our test error using latent variables from $q(z|x)$ for GibbsNet are 25.71\% $\pm$ 0.005 when using the features from the last hidden layer of our encoder. However if we use features from the discriminator (the last convolutional layer and the final fully connected layer) we can obtain 23.23\% $\pm$ 0.006. The ALI model \citep{dumoulin2016ali} reported 19.14\% $\pm$ 0.50 and we found that using the same discriminator features from their model produced an error of 18.5\%. We observe that discriminator feature performed better than latent features consistently throughout training. Also, performance of the latent features would approach that of the discriminator features as training progressed.

\begin{figure}[ht]
\centering
\includegraphics[width=\textwidth]{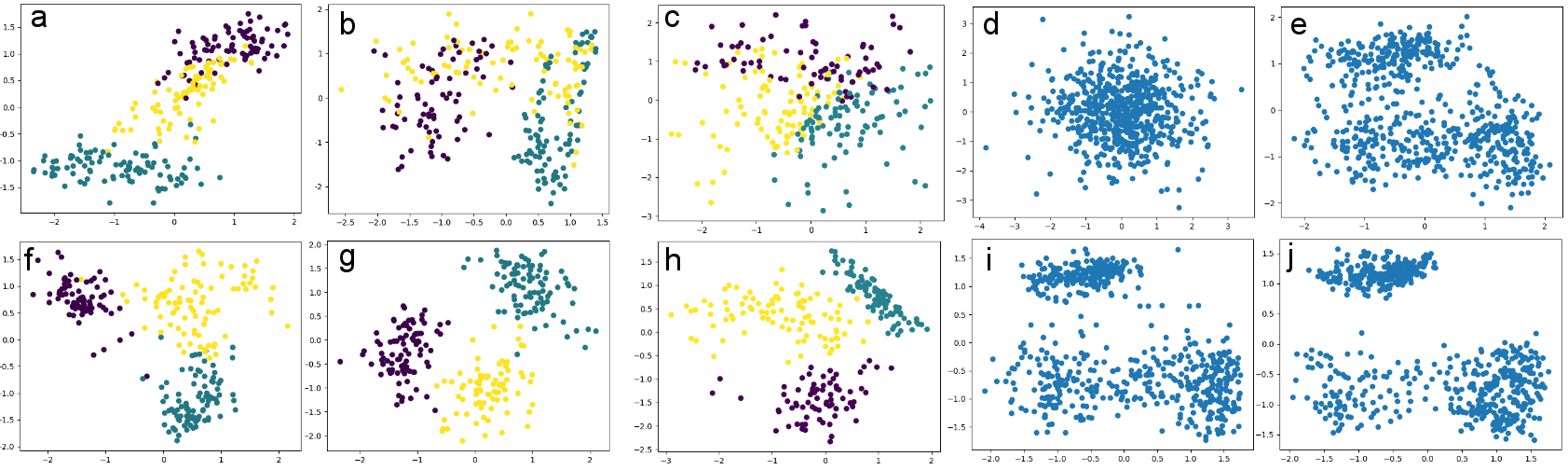}
\caption{Illustration of the distribution over inferred latent variables for real data points from the MNIST digits ($0$, $1$, $9$) learned with different models trained for roughly the same amount of time: GibbsNet with a determinstic decoder and the latent variables not given to the discriminator (a), GibbsNet with a stochastic decoder and the latent variables not given to the discriminator (b), ALI (c), GibbsNet with a deterministic decoder (f), GibbsNet with a stochastic decoder with two different runs (g and h), GibbsNet with a stochastic decoder's inferred latent states in an unclamped chain at $1$, $2$ , $3$, and $15$ steps (d, e, i, and j, respectively) into the P-chain (d, e, i, and j, respectively).  Note that we continue to see refinement in the marginal distribution of z when running for far more steps ($15$ steps) than we used during training ($3$ steps).}
\label{fig:latent_disc}
\end{figure}

%\subsection{Classification with Dynamic Missing Features During Train and Test}

%We demonstrate that by using GibbsNet we can do classification where arbitrary subsets of features are missing during training or testing.  GibbsNet is able to handle features missing during training by making the clamped chain an iterative chain where the observed portion of the image is clamped and the rest of the image (missing) is allowed to evolve based on the learned transition operator.  Then the (x,z) on the final step from this iterative process is given to the discriminator.  Our classifier for the digit is trained on top of the x from the final step of the clamped chain.  For the data, we follow the exact same setup as \citep{tang2010gated} and when processing each example, remove a random block of the image with a width uniformly random from 7 to 16 pixels and a height uniformly random from 7 to 16 pixels, and a random position.  This is done for all images during both \textit{training} and testing, and full images are never available for the model.  

%Deep Boltzmann Machines achieve a test error of $7.49\%$ on this task \citep{tang2010gated}, Gated Boltzmann Machines achieve $6.53\%$, and Deep Belief Networks achieve $8.39\%$.  Using GibbsNet we achieve $TODO FILL\%$ on this task, and using a classifier trained directly on the occluded digits achieves $TODO FILL\%$.  

%\citep{tang2010deep,tang2010gated,sharir2016circuit}

\subsection{Inception Scores}

The GAN literature is limited in terms of quantitative evaluation, with none of the existing techniques (such as inception scores) being satisfactory \citep{theis2015note}.  Nonetheless, we computed inception scores on CIFAR-10 using the standard method and code released from \citet{salimans2016improved}.  
In our experiments, we compared the inception scores from samples from Gibbsnet and ALI on two tasks, generation and inpainting.

Our conclusion from the inception scores (Table~\ref{tb:inception_scores_table}) is that GibbsNet slightly improves sample quality but greatly improves the expressiveness of the latent space z, which leads to more detail being preserved in the inpainting chain and a much larger improvement in inception scores in this setting. The supplementary materials includes examples of sampling and inpainting chains for both ALI and GibbsNet which shows differences between sampling and inpainting quality that are consistent with the inception scores. 

{\renewcommand{\arraystretch}{1.2}
\begin{table}[ht]
\centering
\caption{Inception Scores from different models.  Inpainting results were achieved by fixing the left half of the image while running the chain for four steps.  Sampling refers to unconditional sampling.  }
\begin{tabular}{|l|l|l|l|} \hline
Source & Samples & Inpainting    \\ \hline
Real Images & 11.24 & 11.24    \\ \hline
ALI (ours) & 5.41 & 5.59    \\ \hline
ALI (Dumoulin) & 5.34 & N/A \\ \hline
GibbsNet & 5.69 & 6.15  \\ \hline
\end{tabular}
\label{tb:inception_scores_table}
\end{table}
}

\subsection{Generation, Inpainting, and Learning the Image-Attribute Joint Distribution}

\paragraph{Generation}
Here, we compare generation on the CIFAR dataset against Non-Equilibrium Thermodynamics method \citep{dickstein2015noneq}, which also begins its sampling procedure from noise.
We show in Figure~\ref{fig:cifar_samples} that, even with a relatively small number of steps ($20$) in its sampling procedure, GibbsNet outperforms the Non-Equilibrium Thermodynamics approach in sample quality, even after many more steps ($1000$).

\paragraph{Inpainting}
The inpainting that can be done with the transition operator in GibbsNet is stronger than what can be done with an explicit conditional generative model, such as Conditional GANs, which are only suited to inpainting when the conditioning information is known about during training or there is a strong prior over what types of conditioning will be performed at test time.
We show here that GibbsNet performs more consistent and higher quality inpainting than ALI, even when the two networks share exactly the same architecture for $p(x \mid z)$ and $q(z \mid x)$ (Figure~\ref{fig:test}), which is consistent with our results on latent structure above.  

\paragraph{Joint generation}
Finally, we show that GibbsNet is able to learn the joint distribution between face images and their attributes \citep[CelebA,][]{liu2015faceattributes} (Figure~\ref{fig:attributes}).
In this case, $q(z \mid x, y)$ ($y$ is the attribute) is a network that takes both the image and attribute, separately processing the two modalities before joining them into one network.
$p(x, y \mid z)$ is one network that splits into two networks to predict the modalities separately.
Training was done with continuous boundary-seeking GAN~\citep[BGAN,][]{hjelm2017boundary} on the image side (same as our other experiments) and discrete BGAN on the attribute side, which is an importance-sampling-based technique for training GANs with discrete data.

%We are also able to perform guided image editing by initializing the chain with a real training example then allowing the chain to run while clamping a subset of the attributes.  

\begin{figure}
\centering
\begin{subfigure}{.5\textwidth}
  \centering
  \includegraphics[width=0.92\linewidth,trim={0cm 0cm 0cm 4.4cm},clip]{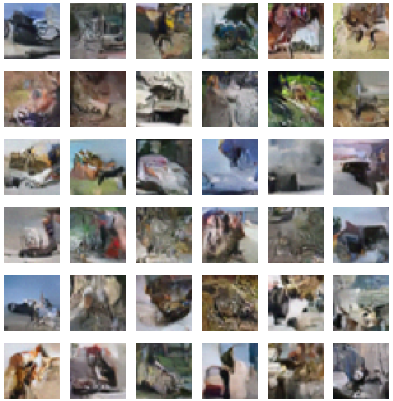}
  \label{fig:cif_sub1}
\end{subfigure}%
\begin{subfigure}{.5\textwidth}
  \centering
  \includegraphics[width=0.95\linewidth,trim={0cm 0cm 0cm 9cm},clip]{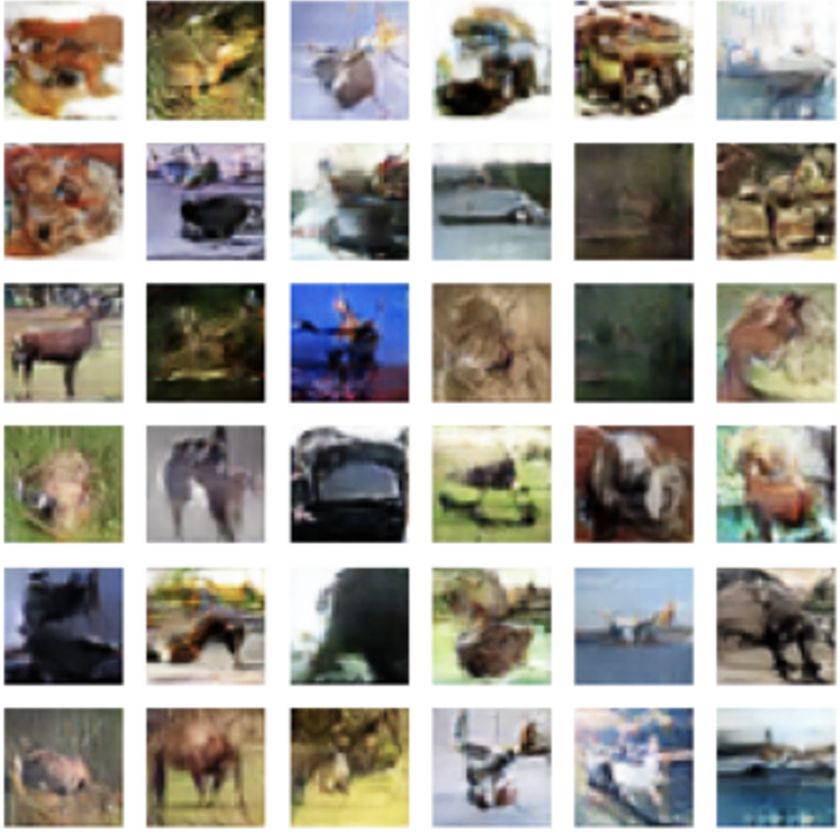}
  \label{fig:cif_sub2}
\end{subfigure}
\caption{CIFAR samples on methods which learn transition operators.  Non-Equilibrium Thermodynamics \citep{dickstein2015noneq} after 1000 steps (left) and GibbsNet after 20 steps (right).  }
\label{fig:cifar_samples}
\end{figure}

%was using 2, 8, 2, 2
\begin{figure}[ht]
\centering
\begin{subfigure}{.5\textwidth}
  \centering
  \includegraphics[width=0.95\linewidth,trim={2cm 2cm 2cm 2cm},clip]{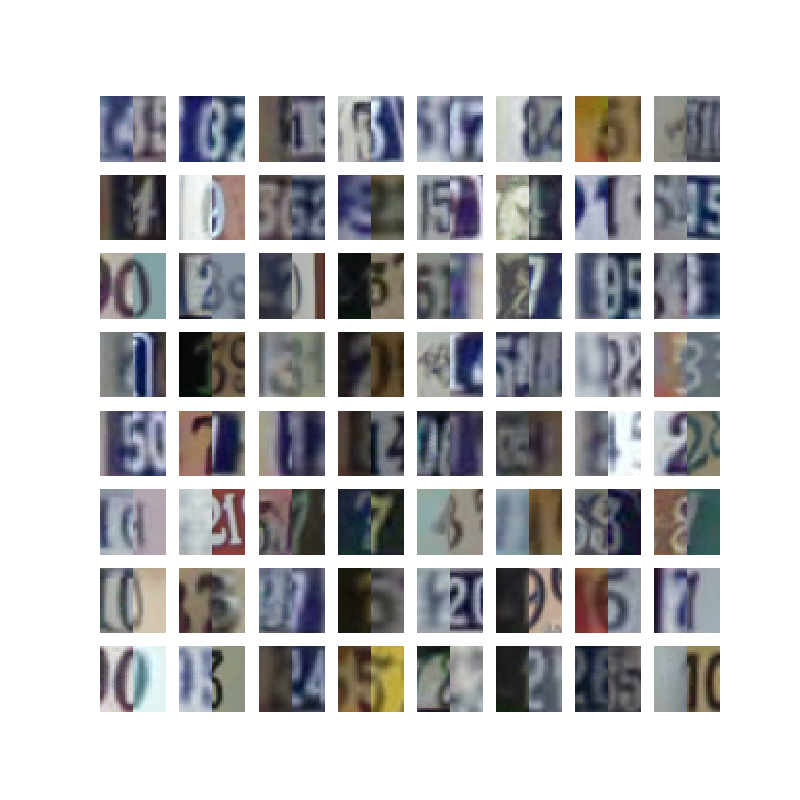}
  \caption{SVHN inpainting after 20 steps (ALI).  }
  \label{fig:sub1}
\end{subfigure}%
\begin{subfigure}{.5\textwidth}
  \centering
  \includegraphics[width=0.95\linewidth,trim={2cm 2cm 2cm 2cm},clip]{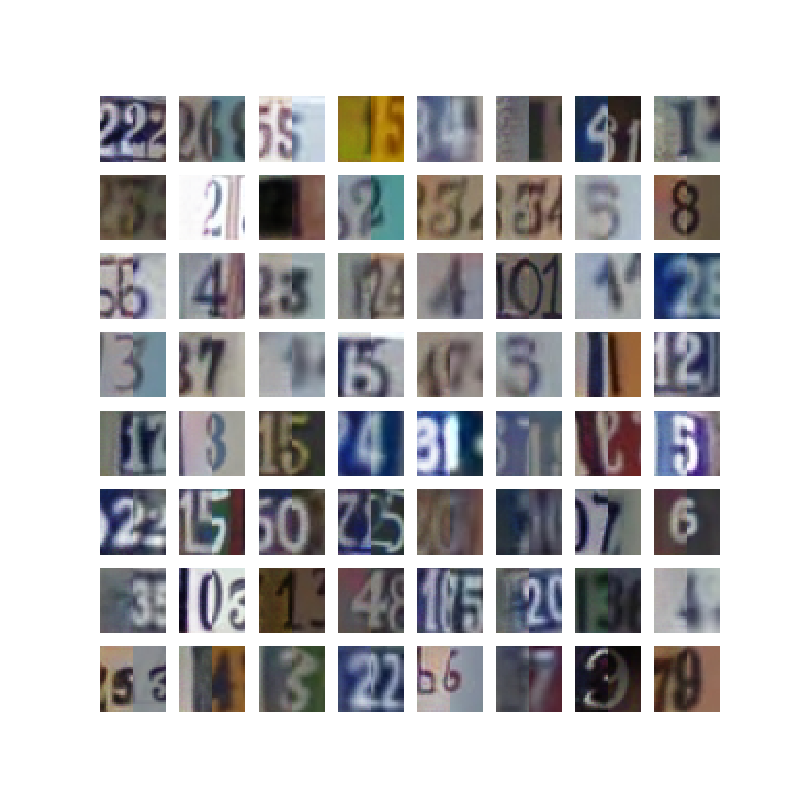}
  \caption{SVHN inpainting after 20 steps (GibbsNet).  }
  \label{fig:sub2}
\end{subfigure}
\caption{Inpainting results on SVHN, where the right side is given and the left side is inpainted.  In both cases our model's trained procedure did not consider the inpainting or conditional generation task at all, and inpainting is done by repeatedly applying the transition operators and clamping the right side of the image to its observed value.  GibbsNet's richer latent space allows the transition operator to keep more of the structure of the input image, allowing for tighter inpainting.}
\label{fig:test}
\end{figure}

\begin{figure}[ht]
\centering
\includegraphics[width=\textwidth]{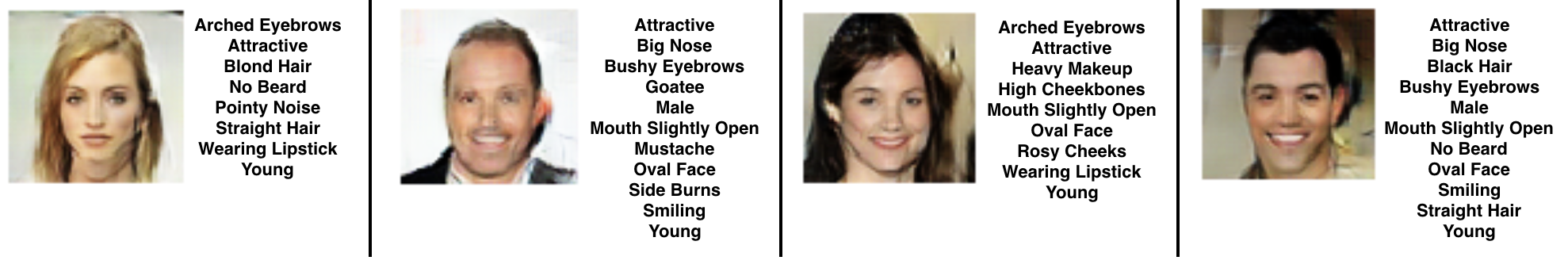}
\caption{Demonstration of learning the joint distribution between images and a list of 40 binary attributes. 
Attributes (right) are generated from a multinomial distribution as part of the joint with the image (left).}
\label{fig:attributes}
\end{figure}

%\begin{figure}[!htb]
%\centering
%\begin{subfigure}{.5\textwidth}
%  \centering
%  \includegraphics[width=0.95\linewidth]{figures/ali_GibbsNet_svhn_recstats5.png}
%  %\includegraphics[width=0.95\linewidth]{inpainting_ali_20.png}
  
%  \label{fig:sub1}
%\end{subfigure}%
%\begin{subfigure}{.5\textwidth}
%  \centering
%  \includegraphics[width=0.95\linewidth,trim={2cm 8cm 2cm 2cm},clip]{18468_inpainting_19.png}
  %\includegraphics[width=0.95\linewidth]{inpainting_ours_20.png}
  
%  \label{fig:sub2}
%\end{subfigure}
%\caption{For various transition steps from x to z to x, the fraction of the SVHN digits which change identity.  In this context reconstruction refers to running the transition operator one step starting from an observed data point from the test set.    }
%\label{fig:test}
%\end{figure}

\section{Conclusion}

We have introduced GibbsNet, a powerful new model for performing iterative inference and generation in deep graphical models.  Although models like the RBM and the GSN have become less investigated in recent years, their theoretical properties worth pursuing, and we follow the theoretical motivations here using a GAN-like objective. 
With a training and sampling procedure that is closely related to undirected graphical models, GibbsNet is able to learn a joint distribution which converges in a very small number of steps of its Markov chain, and with no requirement that the marginal $p(z)$ match a simple prior.  
We prove that at convergence of training, in spite of unrolling only a few steps of the chain during training, we obtain a transition operator whose stationary distribution also matches the data and makes the conditionals $p(x \mid z)$ and $q(z \mid x)$ consistent with that unique joint stationary distribution. 
We show that this allows the prior, $p(z)$, to be shaped into a complicated distribution (not a simple one, e.g., a spherical Gaussian) where different classes have representations that are easily separable in the latent space.  
This leads to improved classification when the inferred latent variables $q(z|x)$ are used directly.  
Finally, we show that GibbsNet's flexible prior produces a flexible model which can simultaneously perform inpainting, conditional image generation, and prediction with a single model not explicitly trained for any of these specific tasks, outperforming a competitive ALI baseline with the same setup. 
%We demonstrate the practical value of this flexibility by training a model which can perform joint image-attribute inpainting: completing an image-attribute pair given an arbitrary part of the image and an arbitrary subset of the attributes.  To our knowledge GibbsNet is the first method to demonstrate this capability.  

\bibliographystyle{apalike}
\bibliography{bibliography}

\end{document}